%% file: main.tex
\documentclass[journal,12pt,onecolumn,draftclsnofoot]{IEEEtran}
\pdfoutput=1
\usepackage{cite}
\usepackage{amsmath}
\usepackage{amsthm}
\usepackage{mathtools,amssymb,amsfonts}
\usepackage{algorithmic}
\usepackage{graphicx}
\usepackage{textcomp}
\usepackage{lipsum}
\usepackage[hidelinks]{hyperref}
\usepackage{pgfplots}
\pgfplotsset{compat=newest}
\usetikzlibrary{external}
\usepackage{ifthen}
\newtheorem{proposition}{Proposition}

\newtheorem{definition}{Definition}

\newtheorem{remark}{Remark}

\newcounter{experiment}[section]
\newenvironment{experiment}[1][]{\refstepcounter{experiment}\par
   \noindent \textbf{Experiment \theexperiment: #1} \rmfamily}{\par}

\def\BibTeX{{\rm B\kern-.05em{\sc i\kern-.025em b}\kern-.08em
    T\kern-.1667em\lower.7ex\hbox{E}\kern-.125emX}}

\allowdisplaybreaks
\begin{document}
\title{Fault Detection via Occupation Kernel Principal Component Analysis}
\author{Zachary Morrison, Benjamin P. Russo, Yingzhao Lian, and Rushikesh Kamalapurkar
\thanks{This research was supported by the Air Force Office of Scientific Research (AFOSR) under contract number FA9550-20-1-0127, the Swiss National Science Foundation (SNSF) under the NCCR Automation project, grant agreement 51NF40\textunderscore180545, and the National Science Foundation (NSF) under awards number 2027999. Any opinions, findings and conclusions or recommendations expressed in this material are those of the author(s) and do not necessarily reflect the views of the sponsoring agencies.\\
This manuscript has been authored by UT-Battelle, LLC, under contract DE-AC05-00OR22725 with the US Department of Energy (DOE). The US government retains and the publisher, by accepting the article for publication, acknowledges that the US government retains a nonexclusive, paid-up, irrevocable, worldwide license to publish or reproduce the published form of this manuscript, or allow others to do so, for US government purposes. DOE will provide public access to these results of federally sponsored research in accordance with the DOE Public Access Plan (\href{http://energy.gov/downloads/doe-public-access-plan}{http://energy.gov/downloads/doe-public-access-plan})}
\thanks{Z. Morrison and R. Kamalapurkar are with the School of Mechanical and Aerospace Engineering, Oklahoma State University, Stillwater, OK, 74074, United States of America (e-mail: {zachmor,rushikesh.kamalapurkar}@okstate.edu).}
\thanks{B. P. Russo is with the Computer Science and Mathematics Division, Oak Ridge National Laboratory, Oak Ridge, TN 37831, United States of America, (e-mail: russobp@ornl.gov) }
\thanks{Y. Lian is with the Automatic Laboratory, Ecole Polytechnique Federale de Lausanne, 1015 Lausanne, Switzerland. (e-mail: yingzhao.lian@epfl.ch).}}

\maketitle

\begin{abstract}
Reliable operation of automatic systems is heavily dependent on the ability to detect faults in the underlying dynamics. While traditional model-based methods have been widely used for fault detection, data-driven approaches have garnered increasing attention due to their ease of deployment and minimal need for expert knowledge. In this paper, we present a novel principal component analysis (PCA) method that uses occupation kernels. Occupation kernels result in feature maps that are tailored to the measured data, have inherent noise-robustness due to the use of integration, and can utilize irregularly sampled system trajectories of variable lengths for PCA. The occupation kernel PCA method is used to develop a reconstruction error approach to fault detection and its efficacy is validated using numerical simulations.
\end{abstract}

\section{Introduction}
\label{sec:Introduction}
Fault detection methods for dynamical systems rely on the identification of anomalous behavior using measured data. Applications of fault detection range from healthcare \cite{SCC.Wang.Zhan.ea2016}; manufacturing \cite{SCC.Wang.Yao2015, SCC.Qin2009}; monitoring sensor behavior \cite{SCC.Pan.Badawi.ea2022, SCC.Navi.Davoodi.ea2015}; monitoring chemical processes \cite{SCC.Choi.Lee.ea2005, SCC.Mansouri.Nounou.ea2016}; identifying the onset of nonlinear behavior in dynamical systems \cite{SCC.Nguyen.Golinval2010}; and identifying traffic anomalies \cite{SCC.Piciarelli.Micheloni.ea2008}. A multitude of approaches to fault detection have been studied over the past few decades, such as data-driven, set-based, observer-based, and time-series analysis methods \cite{SCC.Mu.Yang.ea2022, SCC.Qin2009,SCC.Patton.Chen1997,SCC.Odgaard.Stoustrup2015}. Set-based methods accomplish fault detection by computing a forward reachable set and checking if the system state at the next time step is inside that set\cite{SCC.Mu.Yang.ea2022}. State estimation techniques such as the extended Kalman filter (EKF) and the Leunberger observer have been successfully implemented for fault detection in industrial processes \cite{SCC.Fathi.Ramirez.ea1993, SCC.Patton.Chen1997}. Set-based and observer-based methods are model-based, whereas this paper focuses on data-driven fault detection \cite{SCC.Qin2009}.

Data-driven fault detection methods, such as principal component analysis (PCA), kernelized principal component analysis (KPCA), and the Kahrunen-Loeve transform (KLT) \cite{SCC.Kirlin1999}, typically employ multivariate statistical procedures combined with an index, such as a reconstruction error, Hotteling's $T^{2}$, a squared prediction error (SPE), or a combination thereof, to detect anomalies \cite{SCC.Hoffmann2007, SCC.Mu.Yang.ea2022, SCC.Qin2009}. The KLT utilizes the expansion of a random variable as a linear combination of eigenfunctions of the covariance operator for fault detection \cite{SCC.Odgaard.Stoustrup2015}. In finite dimensions and in the context of data driven methods or discrete sampling, the KLT is simply PCA. 

PCA and KPCA methods extract the principal components of a dynamical system. Specifically, principal component analysis diagonalizes the covariance matrix associated with the fault-free training data. The dominant eigenvectors of the covariance matrix are then used to effectively reduce the dimension of the reconstruction problem by considering only these principal components for reconstruction. A key limitation of PCA is that it fails to capture nonlinearities in the data. Kernelized PCA remedies this limitation by lifting the data to a higher-dimensional feature space via a (nonlinear) feature mapping \cite{SCC.Schoelkopf.Smola.ea1998}. PCA is then applied in the feature space, resulting in non-linear principal components \cite{SCC.Schoelkopf.Smola.ea1998}. Fault detection using KPCA / PCA relies on computation of a metric ($T^2$, SPE, etc.) that measures how well new data can be reconstructed using the principal components \cite{SCC.Wang.Zhan.ea2016,SCC.Pan.Badawi.ea2022,SCC.Choi.Lee.ea2005,SCC.Wang.Yao2015,SCC.Nguyen.Golinval2010,SCC.Navi.Davoodi.ea2015,SCC.Mansouri.Nounou.ea2016,SCC.Hoffmann2007}.


The feature maps used for KPCA are typically the canonical feature maps associated with generic kernel functions such as the Gaussian radial basis function \cite{SCC.Hoffmann2007}. As such, the feature maps are largely independent of the system or the measured data. In this paper, a new PCA framework is developed where the feature maps are also derived from the training data. The idea, motivated by results such as \cite{SCC.Rosenfeld.Kamalapurkar.ea2019a}, is to use trajectories generated by a dynamical system as a fundamental unit of data by embedding them in a reproducing kernel Hilbert space (RKHS) using the so-called occupation kernels.

The resulting PCA method, called occupation kernel PCA (OKPCA), is expected to perform better owing to the use of feature maps that are adapted to the data. In addition the computations required to implement OKPCA rely exclusively on integrals of kernel functions evaluated along system trajectories. As a result, OKPCA is endowed with intrinsic robustness to zero-mean noise and can be implemented on data sets containing variable length trajectories that are irregularly sampled \cite{SCC.Rosenfeld.Kamalapurkar.ea2019a}. Fault detection then proceeds by reconstructing a given trajectory as a linear combination of eigenfunctions of a suitably defined kernelized covariance operator and computing a suitable analog of the reconstruction error used for KPCA by Hoffman \cite{SCC.Hoffmann2007}. 

This paper is organized as follows: Section \ref{Background} establishes the mathematical background of KPCA. Section \ref{OKPCA} outlines the OKPCA method for fault detection. In section \ref{exp}  OKPCA is applied to detect faulty trajectories generated by an academic example and a quadrotor in two numerical experiments. Section \ref{concl} concludes the paper. 

\section{Background} \label{Background}
In this section, a brief overview of current PCA methods is provided for completeness. 
\subsection{Principal Component Analysis}
Given a set of $M$ \emph{centered} observations $\{x_j\in \mathbb{R}^n\}_{j=1}^M\subset X \subseteq \mathbb{R}^n$, where ``centered" indicates that $\sum_{j=1}^M x_j = 0$, the principal component analysis (PCA) procedure diagonalizes the covariance matrix $C$ defined by $C = \frac{1}{M}\sum_{j=1}^M x_{j}x_j^\top$ where $C$ is at most rank $M$ (if all observation vectors are linearly independent). Since $C$ is a positive semi-definite matrix, it is diagonalizable and has non-negative eigenvalues. The eigenvectors of $C$ are referred to as the principal components, typically ordered in a decreasing sequence of the corresponding eigenvalues. Given a vector $v\in \mathbb{R}^n$ we note that $Cv = \frac{1}{M}\sum_{j=1}^M \langle x_j, v\rangle x_j$, where $\langle \cdot, \cdot \rangle$ indicates the standard dot product. In particular, if $v$ is an eigenvector of $C$ with eigenvalue $\lambda$, we have $Cv = \lambda v = \frac{1}{M}\sum_{j=1}^M \langle x_j, v\rangle x_j$, which implies that all eigenvectors of $C$ lie in the span of $\{x_j\}_{j=1}^M$.
\subsection{Kernelized Principal Component Analysis}
Kernelized principal component analysis (KPCA) extends the PCA procedure to produce \emph{nonlinear} principal components. This is done by embedding the data into a reproducing kernel Hilbert space (RKHS) via a feature mapping $\Phi:X\subseteq \mathbb{R}^n \rightarrow H$.
\begin{definition}
Let $X$ be a non-empty set. A function $k:X\times X\rightarrow \mathbb{R}$ is called a \emph{kernel function on $X$} if there exists a $\mathbb{R}$-Hilbert space $H$ and a map $\Phi:X\rightarrow H$ such that for all $x,x'\in X$ we have $k(x,x')=\langle \Phi(x'),\Phi(x)\rangle_H$. We call $\Phi$ a feature map and $H$ a feature space of $k$.
\end{definition}
In other words, the data point $x$ is replaced by a element $\Phi(x)$ in the Hilbert space $H$ and the dot product is replaced by an inner-product over the Hilbert space. It should be noted that the choice of feature map is \emph{not} unique, however, the Moore-Aronszajn theorem guarantees the existence of a unique RKHS corresponding to $k$ and a canonical feature map that maps into that RKHS in the case where $k$ is a positive semidefinite kernel.
\begin{definition}
A RKHS, $H$, over a set $X$ is a Hilbert space of real valued functions over the set $X$ such that for all $x \in X$ the evaluation functional, $E_x: H \to \mathbb{R}$, given as $E_x g := g(x)$ is bounded.
\end{definition}
The Riesz representation theorem guarantees, for all $x \in X$, the existence of a unique function $k_x \in H$ such that $\langle g, k_x \rangle_H = g(x)$, where $\langle \cdot, \cdot \rangle_H$ is the inner product for $H$ \cite[Chapter 1]{SCC.Paulsen.Raghupathi2016}. The function $k_x$ is called the \emph{reproducing kernel at $x$}, the function $k(x,y) = \langle k_y, k_x \rangle_H$ is called the \emph{kernel function corresponding to $H$} and the mapping $\Phi: X \rightarrow H$ given by $x\mapsto k(\cdot,x)=\Phi(x)$, is called the \emph{canonical feature map}. 

In this setting we can now define \emph{nonlinear} principal components via analogous constructions. Given a feature map $\Phi:X\subseteq \mathbb{R}^n\rightarrow H$ and a set of data $\{x_j\}_{j=1}^M$ centered in $H$, i.e. $\sum_{j=1}^M\Phi(x_j) = 0$, the kernelized covariance operator $C:H\rightarrow H$ is defined as $C = \frac{1}{M}\sum_{j=1}^M[\Phi(x_j)\otimes\Phi(x_j)]$,
where, the notation $[u\otimes v]$, for $u,v\in H$, denotes the rank one operator defined by $[u\otimes v]h = \langle h,v\rangle u$ for $h\in H$. 

It is worth noting that $C$ is a finite rank and positive semi-definite operator and thus diagonalizable. If $v$ is an eigenfunction of $C$ then automatically $v\in \text{span}\{\Phi(x_j) : j=1,\ldots, M\}$ and $v=\sum_{j=1}^M \alpha_j \Phi(x_j)$ for $\alpha_j\in \mathbb{R}$. The coefficients $\alpha_i$ can be computed quite easily by solving a matrix equation, indeed for an eigenfunction $v\in H$, $\langle\Phi(x_k), Cv\rangle_H = \langle\Phi(x_k), \lambda v\rangle_H$, which, along with
\[
    \langle\Phi(x_k), \lambda v\rangle_H = \lambda \sum_{i=1}^M \langle \Phi(x_k), \alpha_{i}\Phi(x_i) \rangle,
\]
implies by definition of $\otimes$ that
\begin{equation}\label{eq:matrixFormFeatureMap}
    \langle\Phi(x_k), Cv\rangle_H = \sum_{i,j=1}^M\frac{\alpha_i \langle \Phi(x_j), \Phi(x_i)\rangle_H \langle \Phi(x_k), \Phi(x_j)\rangle_H}{M}.
\end{equation}
If we define $\alpha = (\alpha_1,\ldots, \alpha_M)^\top$, $k(x_i,x_j)=\langle \Phi(x_i), \Phi(x_j)\rangle_H$ and $K= \left(k(x_i,x_j)\right)_{i,j=1}^{M}$, equation \eqref{eq:matrixFormFeatureMap} can be expressed in the matrix form
\begin{equation}\label{eq:matrixFormCoefficients}
    M\lambda K\alpha = K^2\alpha.
\end{equation}
Since $K$ is a positive semi-definite matrix, it is sufficient to solve the equation $ K\alpha = \lambda M \alpha $ to recover all the solutions to \eqref{eq:matrixFormCoefficients}. In other words, the vector of coefficients $\alpha$ is a normalized eigenvector of the matrix $K$.

Let $\alpha^{(1)},\ldots,\alpha^{(N)}$, for $0 < N \leq M$, be a set of eigenvectors of $K$, corresponding to nonzero eigenvalues $0<\lambda_1\leq \ldots, \leq \lambda_N$, normalized such that for $k=1,\ldots, N$, $\lambda_k\langle \alpha^{(k)}, \alpha^{(k)} \rangle_{\mathbb{R}^n} = 1$. The $k-$th eigenfunction $v^{(k)}$ of $C$ can then be expressed as $ v^{(k)}=\sum_{i=1}^M \alpha^{(k)}_i \Phi(x_i)\in H$.
\begin{definition}
 Given a test point $x\in X$, we call $\langle v^{(k)}, \Phi(x)\rangle_H$, where $v^{(k)}$ is an eigenfunction of $C$, a \emph{nonlinear principal component of $\{x_j\}_{j=1}^M$ at $x$ corresponding to $\Phi$}. 
\end{definition}
\begin{remark}
If the data used for PCA are \emph{uncentered} in $H$, they can be centered by replacing $K$ with \[\tilde{K}=K-J_MK - KJ_M +J_MKJ_M,\] where $(J_M)_{i,j}=\frac{1}{M}$.
\end{remark}
\section{Occupation Kernel PCA} \label{OKPCA}
In this section, we will appropriately modify KPCA to incorporate trajectories as a fundamental unit of data. To do so will require an embedding of trajectories into a RKHS. The embedding will be achieved by using the novel occupation kernels developed in \cite{SCC.Rosenfeld.Kamalapurkar.ea2019a}, and the resulting technique will be called OKPCA.

\begin{definition}\label{def:occ}
Let $X \subset \mathbb{R}^n$ be compact, $H$ be a RKHS of real-valued continuous functions over $X$, and $\gamma\in C([0,T],X) $ be a trajectory, where $C([0,T],X)$ denotes the set of continuous functions from $[0,T]$ to $X$. The functional $g \mapsto \int_0^T g(\gamma(\tau)) \mathrm{d}\tau$ is bounded, and may be represented as $\int_0^T g(\gamma(\tau)) \mathrm{d}\tau = \langle g, \Gamma_{\gamma}\rangle_H,$ for some $\Gamma_{\gamma} \in H$ by the Riesz representation theorem. The function $\Gamma_{\gamma}$ is called the \emph{occupation kernel corresponding to $\gamma$ in $H$} \cite{SCC.Rosenfeld.Kamalapurkar.ea2019a}.
\end{definition}
The occupation kernel corresponding to a trajectory can be shown to be the integral of a kernel function along the trajectory.
\begin{proposition}\label{prop:integral-rep}
\cite{SCC.Rosenfeld.Kamalapurkar.ea2019a} Let $H$ be a RKHS of real-valued continuous functions over a set $X$ and let $\gamma : [0,T] \to X$ be a continuous trajectory as in Definition \ref{def:occ}. The occupation kernel corresponding to $\gamma$ in $H$, $\Gamma_{\gamma}$, may be expressed as 
\begin{equation}\label{eq:integral-rep}\Gamma_\gamma(x) = \int_0^T K(x,\gamma(t)) \mathrm{d}t.\end{equation}
\end{proposition}
\begin{proof}
Note that $\Gamma_\gamma(x) = \langle \Gamma_\gamma, K(\cdot,x)\rangle_H$, by the reproducing property of $K$. Consequently,
\begin{align*}
    \Gamma_\gamma(x) &= \langle \Gamma_\gamma, K(\cdot,x)\rangle_H
    = \langle K(\cdot,x), \Gamma_\gamma \rangle_H\\
     &= \int_0^T K(\gamma(t),x)\, \mathrm{d}t
    = \int_0^T K(x,\gamma(t))\, \mathrm{d}t,
\end{align*}
which establishes the result.
\end{proof}
A kernelized covariance operator can now be defined for a set of trajectories.
\begin{definition}
Let $H$ be a Hilbert space, $\Gamma=\{\gamma_i:[0,T]\rightarrow X\}_{i=1}^M$ be a finite set of trajectories and  $\Phi:C([0,T],X)\rightarrow H$ be a feature mapping taking trajectories into $H$. With $\{\Phi(\gamma_j) :j=1,\ldots M\}$ centered in $H$, i.e. $\sum_{j=1}^M\Phi(\gamma_j)=0$, define the kernelized covariance operator as \[ C_{\Gamma}=\frac{1}{M}\sum_{j=1}^M[\Phi(\gamma_j)\otimes \Phi(\gamma_j)].\]

\end{definition}

Similar to the kernelized covariance operator $C$ above, $C_\Gamma$ is a positive semidefinite finite rank operator and as a result, admits eigenfunctions of the form $v^{(k)}=\sum_{i=1}^M\alpha_i^{(k)}\Phi(\gamma_i)$.  The notion of nonlinear principal components then extends naturally to Hilbert spaces.
\begin{definition}
Given a test trajectory $\gamma:[0,T]\rightarrow X$ and a feature map $\Phi:C([0,T],X)\rightarrow H$, we call $\langle v^{(k)}, \Phi(\gamma)\rangle_H$, where $v^{(k)}$ is an eigenfunction of $C_{\Gamma}$, a \emph{nonlinear principal component of $\Gamma$ at $\gamma$ corresponding to $\Phi$}. 
\end{definition}
While the principal components can be defined with respect to any feature map, the occupation kernels themselves provide a feature map that is convenient for analysis and implementation. The convenience stems from the fact that if the occupation kernels are selected to be the feature maps, the coefficients $\alpha^{(k)}_i$ of the eigenfunction $v^{(k)}$ of $C_{\Gamma}$ are given by normalized eigenvectors of the Gram matrix of occupation kernels.
\begin{proposition}
The mapping $\Phi(\gamma) = \Gamma_\gamma$ is a feature map from $C([0,T],X)$ to $H$. The eigenfunctions of $C_\Gamma$ under this feature map can be computed by solving $ \tilde{K}_\Gamma \alpha = \lambda M \alpha $ where $\tilde{K}_\Gamma$ is the \emph{centered} occupation kernel Gram matrix, given by
\[
    \tilde{K}_{\Gamma}=K-J_MK - KJ_M +J_MKJ_M
\]
where $(J_M)_{i,j}=\frac{1}{M}$ and 
$K = \left(\langle\Gamma_{\gamma_i},\Gamma_{\gamma_j}\right\rangle_H)_{i,j=1}^{M}$ is the original occupation kernel Gram matrix. In particular, if $\alpha^{(1)},\ldots,\alpha^{(N)}$, for $0 < N \leq M$, are eigenvectors of $\tilde{K}_{\Gamma}$, corresponding to nonzero eigenvalues $0<\lambda_1\leq \ldots, \leq \lambda_N$, normalized such that for $k=1,\ldots, N$, $\lambda_k \langle \alpha^{(k)}, \alpha^{(k)} \rangle_{\mathbb{R}^n} = 1$, then the $k-$th eigenfunction $v^{(k)}$ of $C_{\Gamma}$ can be expressed as 
\[
    v^{(k)}=\sum_{i=1}^M \alpha^{(k)}_i \Gamma_{\gamma_i}\in H.
\]
\end{proposition}
\begin{proof} The proof of the above proposition proceeds analogously to what is done in KPCA. We need only note that $
    \langle \Phi(\gamma_i),\Phi(\gamma_j)\rangle_H = \langle \Gamma_{\gamma_i},\Gamma_{\gamma_j}\rangle_H$
and that $\tilde{K}_{\Gamma} = (J_M-I)K(J_M-I)$ is positive semi-definite.
\end{proof}
\subsection{OKPCA for Fault Detection}
Here we will outline an interesting application of OKPCA to detect faulty trajectories based on Hoffman's reconstruction error \cite{SCC.Hoffmann2007}. 
\begin{definition}
Let $\gamma$ be a test trajectory, $\Gamma=\{\gamma_j:j=1,\ldots M\}$ be a collection of trajectories, $V=\{v^{(k)}: k = 1,\ldots N\}$ be a collection of eigenfunctions for $C_{\Gamma}$,  and $\Phi_0 = \frac{1}{M}\sum_{j=1}^M \Phi(\gamma_j)$ be the \emph{center} of $\Gamma$ in $H$. Letting $\tilde{\Phi}(\gamma)=\Phi(\gamma)-\Phi_0$ we can define the \emph{reconstruction error for $\gamma$ in $H$ with respect to $V$} by 
\begin{equation}\label{eq:reconstructionError}
    R(\gamma) = \|\tilde{\Phi}(\gamma)\|_H^2-\sum_{j=1}^N\langle{\tilde{\Phi}(\gamma), v^{(j)}\rangle_H^2}.
\end{equation}
\end{definition}

\begin{remark}
If the feature maps are selected to be the occupation kernels, the reconstruction error can be computed using integrals of the kernel function along the trajectory. Indeed, using the feature map $\Phi(\gamma)=\Gamma_\gamma$, we get
\begin{gather*}
\|\tilde{\Phi}(\gamma)\|^2_H=\left\langle \Phi(\gamma)-\sum_{j=1}^M \frac{\Phi(\gamma_j)}{M},\Phi(\gamma)-\sum_{j=1}^M \frac{\Phi(\gamma_j)}{M}\right\rangle_H\\
=\langle
\Gamma_{\gamma},\Gamma_{\gamma}\rangle_H - \sum_{j=1}^M\frac{2\langle \Gamma_\gamma,\Gamma_{\gamma_j}\rangle_H}{M} + \sum_{i,j=1}^M\frac{\langle
\Gamma_{\gamma_i},\Gamma_{\gamma_j}\rangle_H}{M^2}
\end{gather*}
 and for a given $k$ we have 
\begin{align*}&\langle \tilde{\Phi}(\gamma), v^{(k)}\rangle_H \\
&= \sum_{j=1}^M \alpha_j^{(k)}\left[\langle \Gamma_\gamma,\Gamma_{\gamma_j} \rangle_H - \frac{1}{M} \sum_{n=1}^M \langle \Gamma_{\gamma_n}, \Gamma_{\gamma_j}\rangle_H\right.\\
&-\left.\frac{1}{M}\sum_{\ell=1}^M \langle \Gamma_{\gamma}, \Gamma_{\gamma_\ell}\rangle_H+\frac{1}{M^2}\sum_{n,\ell=1}^{M}\langle \Gamma_{\gamma_{n}},\Gamma_{\gamma_\ell}\rangle_H\right].
\end{align*}
The reconstruction error can then be computed using the fact that given two trajectories $\gamma_i$ and $\gamma_j$, the inner product of the corresponding occupation kernels is given by
\[
    \langle \Gamma_{\gamma_i},\Gamma_{\gamma_j}\rangle_H = \int_0^T\int_0^T k(\gamma_i(\tau),\gamma_j(t)) \, \mathrm{d}\tau \mathrm{d}t.
\]
\end{remark}\medskip

\begin{remark}
Similar to KPCA, the OKPCA reconstruction error also has an interesting geometric interpretation. Note that the reconstruction error can be represented in the inner product form 
\[
R(\gamma) = \left\langle \tilde{\Phi}(\gamma) , \tilde{\Phi}(\gamma)- \sum_{j=1}^N \langle\tilde{\Phi}(\gamma), v^{(j)}\rangle_H v^{(j)}\right\rangle_H.
\] Hence, the reconstruction error is a measure of how well the projection of $\tilde{\Phi}(\gamma)$ onto $\text{span}\{v^{(j)}: j=1, \ldots, N\}$ recreates $\tilde{\Phi}(\gamma)$.
\end{remark}
Given a large enough set of normal trajectories, the reconstruction error can thus be used to detect faulty trajectories.
\begin{definition}
Let $\Gamma=\{\gamma_j:j=1,\ldots M\}$ be a collection of trajectories, called \emph{training data}. Let $V=\{v^{(k)}: k = 1,\ldots N\}$ denote the  \emph{principal component vectors}, i.e., a collection of eigenfunctions for $C_\Gamma$ corresponding to non-zero eigenvalues. Let $R_V(\gamma)$ be the reconstruction error  for a test trajectory $\gamma$ in $H$ with respect to $V$. For a given threshold $\varepsilon>0$, we will call a test trajectory \emph{$\varepsilon-$faulty} if $R_V(\gamma)>\varepsilon$. 
\end{definition}
\begin{remark}
This definition of fault is dependent on $N$, the number of principal component vectors being used to compute the reconstruction error, the selected kernel, and the threshold $\varepsilon$. The threshold $\varepsilon$ can be decided based on reconstruction errors evaluated at trajectories that are a part of the training data. For further remarks on the selection of the kernel and the number of principal component vectors, see the discussion section.
\end{remark}

\section{Experiments} \label{exp}
In the following, two numerical experiments are presented to illustrate the efficacy of the developed fault detection method. The first experiment is an academic one where the developed method is used to identify trajectories generated by a nonlinear system that is different from the one used to generate the training data.

In the second experiment, simulated trajectories of a quadrotor aircraft are used to train the algorithm. The trained algorithm is then used to identify trajectories generated by a faulty quadrotor, where the fault is introduced by changing control parameters.
\subsection{Description and Results}
\begin{experiment}\label{exp:academic}
In this experiment, 100 fault detection trials are performed. In each trial, the training data comprises of 100 trajectories of the system
\begin{align*}
    \dot{x}_{1} = -x_{1} + x_2 \sin\left(\frac{\pi x_1}{2}\right),
    \dot{x}_{2} = -x_{2} + x_1 \cos\left(\frac{\pi x_1}{2}\right)
\end{align*}
initialized from randomly selected initial conditions on the unit circle. To test the developed OKPCA fault detection method, the reconstruction error is evaluated at 20 trajectories of the same system and 20 trajectories of the \emph{faulty} system 
\begin{align*}
    \dot{x}_{1} = -x_{1} + 0.9 x_2 \sin\left(\frac{\pi x_1}{5}\right),
    \dot{x}_{2} = -x_{2} + 0.8 x_1 \cos\left(\frac{\pi x_2}{3}\right)
\end{align*}
also starting from random initial conditions on the unit circle. 

All trajectories are 2 seconds long and sampled every 0.01 seconds. The Gaussian radial basis function $k(x,y) = \mathrm{e}^{\frac{-\left\Vert x-y\right\Vert^2}{\mu}}$ is used as the kernel function with width parameter $\mu=0.6$ and $N=20$ eigenvectors are selected for the projection in \eqref{eq:reconstructionError}. The detection threshold is set to be equal to 2 times the highest reconstruction error seen in the training data, that is, $\varepsilon = 2\max_{i}\{R(\gamma_i)\}_{i=1}^M$. Normal test trajectories with reconstruction errors higher than the threshold are classified as false positives and faulty test trajectories with reconstruction errors smaller than the threshold are classified as false negative. To compare OKPCA and KPCA in a way that is independent of threshold selection, a \emph{mixing percentage} is computed. The mixing percentage is defined as the percentage of the test trajectories that fall within the band defined by the smallest reconstruction error among faulty trajectories and the largest reconstruction error among normal trajectories. The performance of OKPCA and KPCA for this test is summarized in the third column of Table \ref{tab:comparison}

\begin{table}
    \centering
    \bgroup
    \setlength\tabcolsep{2.5pt}
    \small
    \begin{tabular}{|c|c|ccc|ccc|ccc|}
        \hline
         $M$ & Method & \multicolumn{3}{c|}{No Noise (\%)} & \multicolumn{3}{c|}{Samp. Noise (\%)}  & \multicolumn{3}{c|}{Meas. Noise (\%)}   \\
        \hline
         & & FP & FN & MP & FP & FN & MP & FP & FN & MP  \\
         \hline 
         50 & OKPCA  & 11.5 & 0.1 & 2.6  & 10.2 & 0.2 & 4  & 11.2 & 0.7 & 11.9 \\
             & KPCA   & 10.6 & 0.1 & 1.9  & 15 & 0.3 & 8.1  & 11.9 & 1.8 & 23.1 \\
        \hline 
         100 & OKPCA  & 1.3 & 0.2 & 0.7  & 1.3 & 0.6 & 3.9  & 0.6 & 1.8 & 8.8 \\
             & KPCA   & 1.6 & 0.1 & 0.9  & 1.7 & 0.7 & 5.5  & 0.3 & 3.2 & 12.4 \\
        \hline
        150 & OKPCA  & 0.2 & 0.2 & 0  & 0.1 & 0.5 & 3.6  & 0.2 & 1.9 & 8.2 \\
            & KPCA   & 0.4 & 0.1 & 0.5  & 0.1 & 0.8 & 5.2  & 0.1 & 3.7 & 12.6 \\
        \hline
    \end{tabular}
    \egroup
    \caption{A comparison of OKPCA with KPCA for the system and fault models in Experiment \ref{exp:academic}. The initialisms FP, FN, and MP denote the false positive rate, the false negative rate, and the mixing percentage, averaged over 100 trials, respectively.}
    \label{tab:comparison}
\end{table}

Since the OKPCA method relies on integrals of trajectories, the data do not need to be equally spaced. To demonstrate the applicability of the OKPCA method to data sets with variable sampling rates, a sampling noise, uniformly distributed in the interval $[-0.004,0.004]$ is added to each sampling instant of the training data and the test data (i.e., the sampling rate is uniformly distributed between 0.002s and 0.01s). The performance of OKPCA and KPCA for this test is summarized in the fourth column of Table \ref{tab:comparison}.

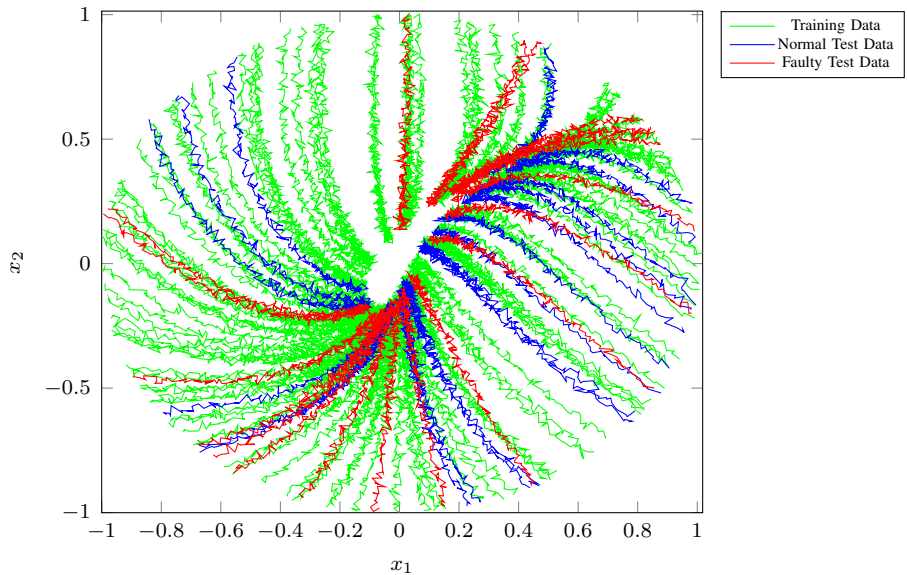
\begin{figure}
    \centering
    \input{figures/academicData}
    \caption{Noisy trajectories used as training data (green) and normal (blue) and faulty (red) test data to test degradation of performance in Experiment \ref{exp:academic}.}
    \label{fig:academicNoisyData}
\end{figure}

As opposed to PCA, which is generally not robust to noise \cite{SCC.Lu.Zhang.ea2004}, occupation kernel PCA, owing to integration of the trajectories, is expected to have inherent robustness to zero-mean measurement noise and sampling noise. To test this hypothesis, the 100 trials are repeated with $M=50$, $100$, and $150$ by adding Gaussian noise with standard deviation 0.01 to each measurement in the training data and the test data (see Fig. \ref{fig:academicNoisyData}). For comparison, the KPCA fault detection method from \cite{SCC.Hoffmann2007} is applied to the same data set with 20 eigenvectors and $\mu = 5$. The performance of OKPCA and KPCA for this test is summarized in the last column of Table \ref{tab:comparison}.

Fig. \ref{fig:academicGoodTrial} illustrates the results of one of the \emph{successful}  (no false positives or false negatives) noisy trials where it can be seen that the faulty test trajectories have a higher reconstruction error than the normal test trajectories. Fig. \ref{fig:academicBadTrial} illustrates the results of one of the \emph{unsuccessful} noisy trials where the decision boundary is not as clear as the successful trial.

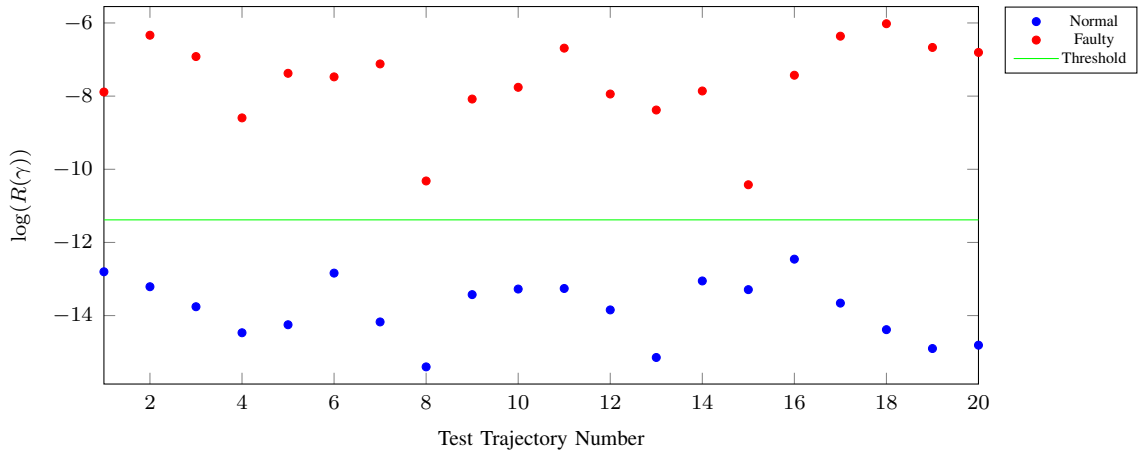
\begin{figure}
    \centering
    \input{figures/academicGoodTrial}
    \caption{An example trial in Experiment \ref{exp:academic} where the faulty trajectories and the normal trajectories are well-separated by the reconstruction error and no false negative or false positive results are generated.}
    \label{fig:academicGoodTrial}
\end{figure}

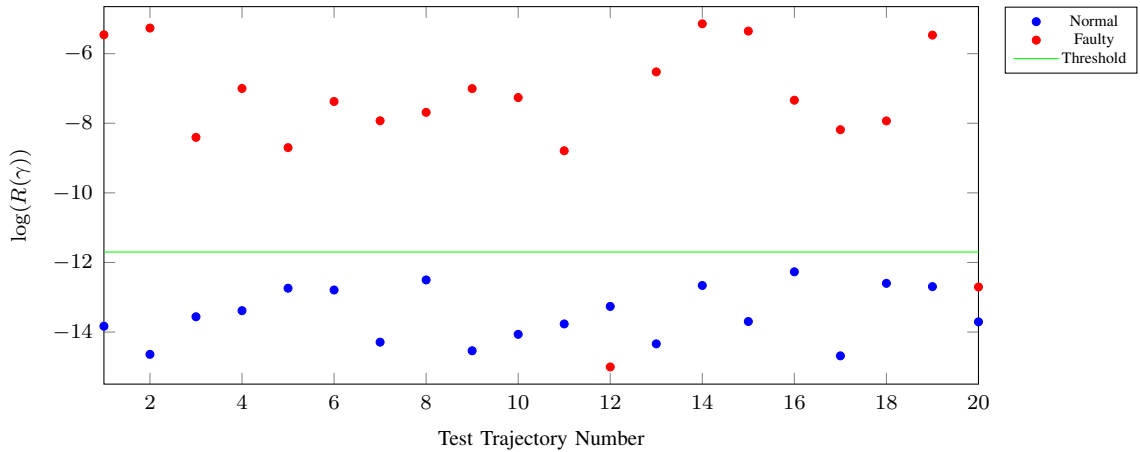
\begin{figure}
    \centering
    \input{figures/academicBadTrial}
    \caption{An example trial in Experiment \ref{exp:academic} where a few of the faulty trajectories fall below the threshold, generating false negative results.}
    \label{fig:academicBadTrial}
\end{figure}
\end{experiment}
\begin{experiment}\label{exp:quadrotor}
In the second experiment, the fault detection capabilities of OKPCA are evaluated using trajectories generated by a quadrotor. A quadrotor model under a known PID controller is simulated in MATLAB. A simplified model of the quadrotor in the vehicle frame is used by neglecting the Coriolis force and assuming the pitch ($\theta$) and roll ($\phi$) angles are small (see Equations $35-40$ in \cite{SCC.Beard2008} for details). The model consists of $12$ state variables that include position $(x,y,z)$, velocity $(u,v,w)$, Euler angles $(\phi,\theta,\psi)$, and roll rates $(p,q,r)$ of the quadrotor.

\begin{figure}
    \centering
    \input{figures/Major_NormAndAnom_Trajectory}
    \caption{Example of a normal (solid) and a faulty (dotted) trajectory of the quadrotor in Experiment \ref{exp:quadrotor} under simulated major actuator fault.}
    \label{fig:NaATrajectory}
\end{figure}
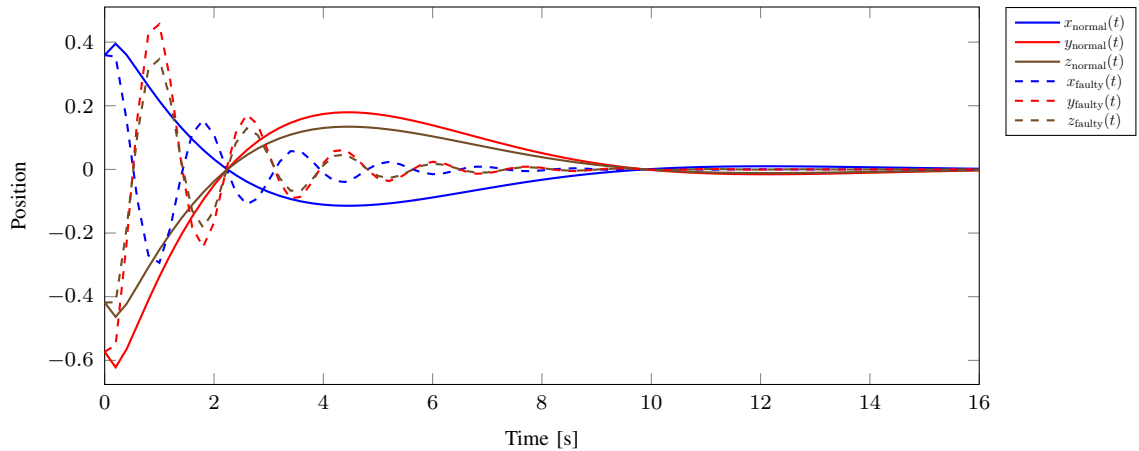

\begin{figure}
    \centering
    \input{figures/Exp2MajorBestTrial}
    \caption{Reconstruction error comparison for major actuator faults in Experiment \ref{exp:quadrotor}.}
    \label{fig:MajWNoise}
\end{figure}
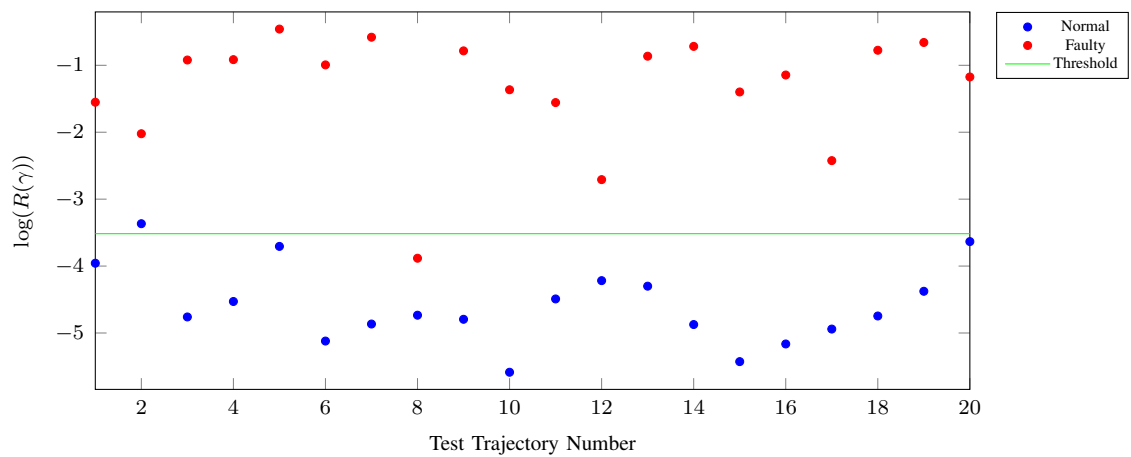

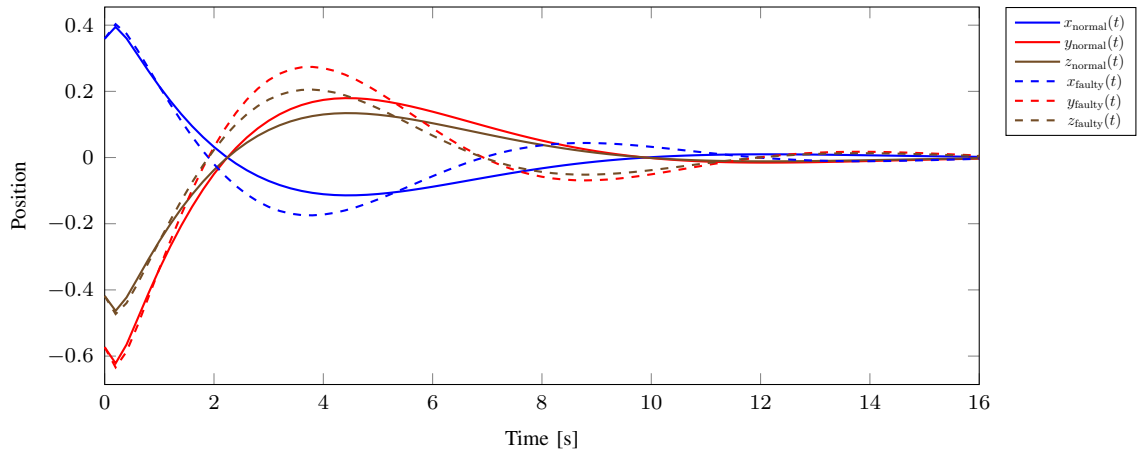
\begin{figure}
    \centering
    \input{figures/Minor_NormAndAnom_Trajectory}
    \caption{Example of a normal (solid) and a faulty (dotted) trajectory of the quadrotor in Experiment \ref{exp:quadrotor} under simulated minor actuator fault.}
    \label{fig:NaATrajectoryMinor}
\end{figure}

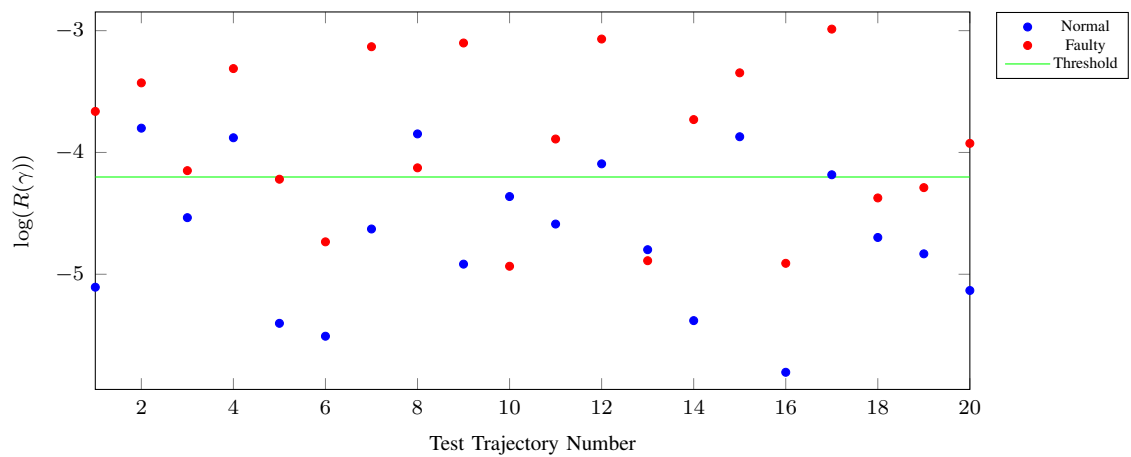
\begin{figure}
    \centering
    \input{figures/Exp2MinorBestTrial}
\caption{Reconstruction error comparison for minor actuator faults in Experiment \ref{exp:quadrotor}.}
   \label{fig:MinWNoise}
\end{figure}

The controller used in the simulation is from Sections $7, 7.2,$ and $7.3$ in \cite{SCC.Beard2008}. Given a desired setpoint, the controller regulates the quadrotor to the setpoint by manipulating the velocity, pitch, and roll using three separate proportional-integral-derivative (PID) controllers. The control gains are identical for each of the three PID controllers. For the training data the proportional gain $K_P$, integral gain $K_I$, and derivative gain $K_D$ were selected to be $5,2$, and $8$, respectively.  For examples of noise-free normal and faulty trajectories for the major and the minor faults, see Figs. \ref{fig:NaATrajectory} and \ref{fig:NaATrajectoryMinor}, respectively.

The algorithm is trained on a data set consisting of $500$ trajectories of randomly generated lengths, sampled at approximately $5$ Hz. Irregular sampling rates and measurement noise are implemented similar to Experiment \ref{exp:academic}. Each trajectory is started from a random initial condition in the box with side length 2 centered at the origin in $\mathbb{R}^{12}$ and the quadrotor is commanded to fly to the origin. Actuator faults are simulated by altering the PID gains. To simulate major actuator faults, $20$ trajectories are generated using $K_P=15$, $K_D=2$, and $K_I=12$, and the minor actuator faults are simulated by generating another $20$ trajectories using $K_P=4$, $K_D=7$, and $K_I=3$.

The Gaussian radial basis function kernel with width parameter $\mu=10$ is used for OKPCA and $N=100$ eigenvectors are used for reconstruction. The reconstruction errors for the faulty trajectories are then compared with those corresponding to $20$ newly generated normal trajectories. Fig. \ref{fig:MajWNoise} and Fig. \ref{fig:MinWNoise} show the fault detection capabilities for trajectories generated with major and minor actuator faults, respectively. The fault detection threshold is set to be $\varepsilon_{major} = 2\max_{i}\{R(\gamma_i)\}_{i=1}^M$ for the major actuator tests and $\varepsilon_{minor} = \max_{i}\{R(\gamma_i)\}_{i=1}^M$ for the minor actuator tests.
\end{experiment}

\subsection{Discussion}

The experiments demonstrate the efficacy of OKPCA for data-driven fault detection applications. As noted in Experiment 1, in randomized trials, without any knowledge of the system model or the fault, OKPCA results in reconstruction errors that differentiate faulty trajectories from normal trajectories with less than 1\% false positive and false negative rates, with moderate degradation in performance when the data and the sampling rates are corrupted with noise. In addition to the practical advantages of OKPCA over KPCA listed in the introduction, Table \ref{tab:comparison} also indicates that in most experiments, OKPCA outperforms KPCA in the mixing percentage metric. The false positive and false negative rates depend on the selected threshold, and as such are not suitable for use as a metric for comparison.

The results of Experiment 2 indicate that OKPCA can detect faulty trajectories, irrespective of measurement noise and sampling noise. While major actuator faults are detectable with high confidence (Fig. \ref{fig:MajWNoise}), minor actuator faults were hard to detect (Fig. \ref{fig:MinWNoise}). Degradation of performance with decreasing severity of faults is expected in data-driven fault detection methods, especially in the presence of measurement noise.

Similar to Hoffman's observations in \cite{SCC.Hoffmann2007}, too small values of the kernel width $\mu$ result in the kernel functions that are near zero everywhere, rendering PCA meaningless. Too large values of $\mu$ result in a near-zero reconstruction error for all trajectories, faulty and normal. In Experiment 1, a large range of values of $\mu$, between 0.6 and 600, was found to yield similar performance. While large, the acceptable range of values of $\mu$  depends, in ways that are not well-understood, on density and number of trajectories in the training data. Selection of $\mu$ can be done using trial and error given a set of trajectories that are known to be faulty. The number of eigenvectors, $N$, needs to be selected large enough to ensure that the reconstruction errors are near zero when evaluated at trajectories in the training data. 

The results in Table \ref{tab:comparison} strongly indicate that larger data sets can result in fewer false positives when fault detection is performed using OKPCA. While the false negative rate is small, it shows no such trend. It should be noted that the errors in Table \ref{tab:comparison} are computed with the threshold in each trial selected as $\varepsilon = 2\max_{i}\{R(\gamma_i)\}_{i=1}^M$. The fact that the false positive rate drops to zero when a larger training data set is used implies that as the training data set gets larger the threshold $\epsilon$ could potentially be selected to be smaller. The authors hypothesize that, with a more judicious selection of the threshold, the decreasing trend in false positive rates, observed in Table \ref{tab:comparison}, can also be realized in the false negative rates, up to a limit, as the training data set gets larger. It should be noted, however, that OKPCA fault detection scales cubically in $M$, and as such, the use of large training data sets requires significant amount of computational resources. 
 
\section{Conclusion} \label{concl}
In this paper, the kernel PCA method is generalized to kernelized covariance operators on reproducing Kernel Hilbert spaces. The resulting OKPCA method generates principal components of a set of \emph{trajectories} as opposed to a set of points. It is shown that when occupation kernels are used as feature maps, the computations involved reduce to computation of single and double integrals of kernel functions along the trajectories in the training data and the test data. The developed OKPCA method is applied to the data-driven fault detection problem to separate normal trajectories of a dynamical system from faulty ones, without any knowledge of the system dynamics. Two numerical experiments demonstrate the efficacy of the developed technique.

The numerical experiments indicate that provided a training data set of known normal trajectories and a test data set of known faulty trajectories is available, the parameters of the developed OKPCA fault detection method can be selected by trial and error from a wide range of acceptable values. Performance improvement with increasing amount of training data is also observed, albeit accompanied by a significant rise in computation costs. The numerical experiments also indicate an inherent robustness to noise. A theoretical analysis of noise-robustness is out of the scope of this paper, and a part of future research. 

\bibliographystyle{ieeeTRAN}
\bibliography{sccmaster,extra,scc}

\end{document}

%% file: figures/academicData.tex
\begin{tikzpicture}
    \begin{axis}[
        xlabel={$x_1$},
        ylabel={$x_2$},
        legend pos = outer north east,
        legend style={nodes={scale=0.5, transform shape}},
        enlarge y limits=0,
        enlarge x limits=0,
        height = 0.5\columnwidth,
        height = 0.5\columnwidth,
        label style={font=\scriptsize},
        tick label style={font=\scriptsize}
    ]
        \pgfplotsinvokeforeach{0,...,99}{
            \pgfmathtruncatemacro\result{100+#1}
            \ifthenelse{\equal{#1}{0}}{
                \addplot [green, mark=none] table[x index=#1, y index=\result] {data/Exp1TrainNoisy.dat};
                \addlegendentry{Training Data}
            }{
                \addplot [green, mark=none, forget plot] table[x index=#1, y index=\result] {data/Exp1TrainNoisy.dat};
            }
        }
        \pgfplotsinvokeforeach{0,...,19}{
            \pgfmathtruncatemacro\result{20+#1}
            \ifthenelse{\equal{#1}{0}}{
                \addplot [blue, mark=none] table[x index=#1, y index=\result] {data/NormalExp1TrainNoisy.dat};
                \addlegendentry{Normal Test Data}
            }{
                \addplot [blue, mark=none, forget plot] table[x index=#1, y index=\result] {data/NormalExp1TrainNoisy.dat};
            }
        }
        \pgfplotsinvokeforeach{0,...,19}{
            \pgfmathtruncatemacro\result{20+#1}
            \addplot [red, mark=none] table[x index=#1, y index=\result] {data/FaultyExp1TrainNoisy.dat};
            \ifthenelse{\equal{#1}{0}}{\addlegendentry{Faulty Test Data}}{}
        }
    \end{axis}
\end{tikzpicture}

%% file: figures/academicGoodTrial.tex
\begin{tikzpicture}
    \begin{axis}[
        xlabel={Test Trajectory Number},
        ylabel={$\log(R(\gamma))$},
        legend pos = outer north east,
        legend style={nodes={scale=0.5, transform shape}},
        enlarge y limits=0.05,
        enlarge x limits=0,
        height = 0.4\columnwidth,
        width = 0.8\columnwidth,
        label style={font=\scriptsize},
        tick label style={font=\scriptsize}
    ]
        \addplot [scatter, only marks, mark options = {blue, scale = 0.75}, scatter/use mapped color={draw=blue, fill=blue}] table [x index=0, y index=1] {data/Exp1BestTrial.dat};
        \addplot [scatter, only marks, mark options = {red, scale = 0.75}, scatter/use mapped color={draw=red, fill=red}] table [x index=0, y index=2] {data/Exp1BestTrial.dat};
        \addplot [green] table [x index=0, y index=3] {data/Exp1BestTrial.dat};
        \legend{Normal,Faulty,Threshold}
    \end{axis}
\end{tikzpicture}

%% file: figures/academicBadTrial.tex
\begin{tikzpicture}
    \begin{axis}[
        xlabel={Test Trajectory Number},
        ylabel={$\log(R(\gamma))$},
        legend pos = outer north east,
        legend style={nodes={scale=0.5, transform shape}},
        enlarge y limits=0.05,
        enlarge x limits=0,
        height = 0.4\columnwidth,
        width = 0.8\columnwidth,
        label style={font=\scriptsize},
        tick label style={font=\scriptsize}
    ]
        \addplot [scatter, only marks, mark options = {blue, scale = 0.75}, scatter/use mapped color={draw=blue, fill=blue}] table [x index=0, y index=1] {data/Exp1WorstTrial.dat};
        \addplot [scatter, only marks, mark options = {red, scale = 0.75}, scatter/use mapped color={draw=red, fill=red}] table [x index=0, y index=2] {data/Exp1WorstTrial.dat};
        \addplot [green] table [x index=0, y index=3] {data/Exp1WorstTrial.dat};
        \legend{Normal,Faulty,Threshold}
    \end{axis}
\end{tikzpicture}

%% file: figures/Major_NormAndAnom_Trajectory.tex
\begin{tikzpicture}
    \begin{axis}[
        xlabel={Time [s]},
        ylabel={Position},
        legend pos = outer north east,
        legend style={nodes={scale=0.5, transform shape}},
        enlarge y limits=0.05,
        enlarge x limits=0,
        height = 0.4\columnwidth,
        width = 0.8\columnwidth,
        label style={font=\scriptsize},
        tick label style={font=\scriptsize}
    ]
        \pgfplotsinvokeforeach{1,...,3}{
            \addplot+ [thick, mark=none] table [x index=0, y index=#1]{data/Exp2_Normal_Major_TCompare.dat};
        }
        \pgfplotsset{cycle list shift=1}
        \pgfplotsinvokeforeach{4,...,6}{
            \addplot+ [thick, dashed, mark=none] table [x index=0, y index=#1]{data/Exp2_Normal_Major_TCompare.dat};
        }
         \legend{$x_{\mathrm{normal}}(t)$,$y_{\mathrm{normal}}(t)$,$z_{\mathrm{normal}}(t)$,$x_{\mathrm{faulty}}(t)$,$y_{\mathrm{faulty}}(t)$,$z_{\mathrm{faulty}}(t)$}
    \end{axis}
\end{tikzpicture}

%% file: figures/Exp2MajorBestTrial.tex
\begin{tikzpicture}
    \begin{axis}[
        xlabel={Test Trajectory Number},
        ylabel={$\log(R(\gamma))$},
        legend pos = outer north east,
        legend style={nodes={scale=0.5, transform shape}},
        enlarge y limits=0.05,
        enlarge x limits=0,
        height = 0.4\columnwidth,
        width = 0.8\columnwidth,
        label style={font=\scriptsize},
        tick label style={font=\scriptsize}
    ]
        \addplot [scatter, only marks, mark options = {blue, scale = 0.75}, scatter/use mapped color={draw=blue, fill=blue}] table [x index=0, y index=1] {data/Exp2MajorBestTrial.dat};
        \addplot [scatter, only marks, mark options = {red, scale = 0.75}, scatter/use mapped color={draw=red, fill=red}] table [x index=0, y index=2] {data/Exp2MajorBestTrial.dat};
        \addplot [green] table [x index=0, y index=3] {data/Exp2MajorBestTrial.dat};
        \legend{Normal,Faulty,Threshold}
    \end{axis}
\end{tikzpicture}

%% file: figures/Minor_NormAndAnom_Trajectory.tex
\begin{tikzpicture}
    \begin{axis}[
        xlabel={Time [s]},
        ylabel={Position},
        legend pos = outer north east,
        legend style={nodes={scale=0.5, transform shape}},
        enlarge y limits=0.05,
        enlarge x limits=0,
        height = 0.4\columnwidth,
        width = 0.8\columnwidth,
        label style={font=\scriptsize},
        tick label style={font=\scriptsize}
    ]
    \pgfplotsinvokeforeach{1,...,3}{
        \addplot+ [thick, mark=none] table [x index=0, y index=#1]{data/Exp2_Normal_Minor_TCompare.dat};
    }
    \pgfplotsset{cycle list shift=1}
    \pgfplotsinvokeforeach{4,...,6}{
        \addplot+ [thick, dashed, mark=none] table [x index=0, y index=#1]{data/Exp2_Normal_Minor_TCompare.dat};
    }
     \legend{$x_{\mathrm{normal}}(t)$,$y_{\mathrm{normal}}(t)$,$z_{\mathrm{normal}}(t)$,$x_{\mathrm{faulty}}(t)$,$y_{\mathrm{faulty}}(t)$,$z_{\mathrm{faulty}}(t)$}
    \end{axis}
\end{tikzpicture}

%% file: figures/Exp2MinorBestTrial.tex
\begin{tikzpicture}
    \begin{axis}[
        xlabel={Test Trajectory Number},
        ylabel={$\log(R(\gamma))$},
        legend pos = outer north east,
        legend style={nodes={scale=0.5, transform shape}},
        enlarge y limits=0.05,
        enlarge x limits=0,
        height = 0.4\columnwidth,
        width = 0.8\columnwidth,
        label style={font=\scriptsize},
        tick label style={font=\scriptsize}
    ]
        \addplot [scatter, only marks, mark options = {blue, scale = 0.75}, scatter/use mapped color={draw=blue, fill=blue}] table [x index=0, y index=1] {data/Exp2MinorBestTrial.dat};
        \addplot [scatter, only marks, mark options = {red, scale = 0.75}, scatter/use mapped color={draw=red, fill=red}] table [x index=0, y index=2] {data/Exp2MinorBestTrial.dat};
        \addplot [green] table [x index=0, y index=3] {data/Exp2MinorBestTrial.dat};
        \legend{Normal,Faulty,Threshold}
    \end{axis}
\end{tikzpicture}